\documentclass[twoside]{article}

\usepackage{aistats2021}
\usepackage[utf8]{inputenc}
\usepackage{amsthm}
\usepackage{amsmath}
\usepackage{amssymb}
\usepackage{esvect}
\usepackage{tikz}
\usepackage{graphicx}
\usepackage{booktabs} 
\usepackage{float}
\usepackage{pgfplots}
\usepackage{algorithm2e}

\newtheorem{theorem}{Theorem}[section]

\newtheorem{lemma}[theorem]{Lemma}

\begin{document}

% If your paper is accepted and the title of your paper is very long,
% the style will print as headings an error message. Use the following
% command to supply a shorter title of your paper so that it can be
% used as headings.
%
%\runningtitle{I use this title instead because the last one was very long}

% If your paper is accepted and the number of authors is large, the
% style will print as headings an error message. Use the following
% command to supply a shorter version of the authors names so that
% they can be used as headings (for example, use only the surnames)
%
%\runningauthor{Surname 1, Surname 2, Surname 3, ...., Surname n}

\twocolumn[

\aistatstitle{A Uniformly Consistent Estimator of non-Gaussian Causal Effects Under the \textit{k}-Triangle-Faithfulness Assumption}

\aistatsauthor{ Shuyan Wang \And Peter Spirtes }

\aistatsaddress{ shuyanw@andrew.cmu.edu \And  ps7z@andrew.cmu.edu } ]

\begin{abstract}
 Kalisch and B\"{u}hlmann (2007) showed that for linear Gaussian models, under the Causal Markov Assumption, the Strong Causal Faithfulness Assumption, and the assumption of causal sufficiency, the PC algorithm is a uniformly consistent estimator of the Markov Equivalence Class of the true causal DAG for linear Gaussian models; it follows from this that for the identifiable causal effects in the Markov Equivalence Class, there are uniformly consistent estimators of causal effects as well. The $k$-Triangle-Faithfulness Assumption is a strictly weaker assumption that avoids some implausible implications of the Strong Causal Faithfulness Assumption and also allows for uniformly consistent estimates of Markov Equivalence Classes (in a weakened sense), and of identifiable causal effects. However, both of these assumptions are restricted to linear Gaussian models. We propose the Generalized $k$-Triangle Faithfulness, which can be applied to any smooth distribution. In addition, under the Generalized $k$-Triangle Faithfulness Assumption, we describe the Edge Estimation Algorithm that provides uniformly consistent estimates of causal effects in some cases (and otherwise outputs ``can't tell"), and the \textit{Very Conservative }$SGS$ Algorithm that (in a slightly weaker sense) is a uniformly consistent estimator of the Markov equivalence class of the true DAG.
\end{abstract}

\section{Introduction}
 It has been proved that under the Causal Markov, Faithfulness assumptions and Causal Sufficiency Assumption, there are no uniformly consistent estimators of Markov equivalence classes of causal structures represented by directed acyclic graphs (DAG)(Robins  et  al.,  2003). Kalisch and B\"{u}hlmann (2007) showed that for linear Gaussian models, under the Causal Markov Assumption, the Strong Causal Faithfulness Assumption, and the assumption of causal sufficiency, the PC algorithm is a uniformly consistent estimator of the Markov Equivalence Class of the true causal DAG for linear Gaussian models; it follows from this that for the identifiable causal effects in the Markov Equivalence Class, there are uniformly consistent estimators of causal effects as well. The $k$-Triangle-Faithfulness Assumption is a strictly weaker assumption that avoids some implausible implications of the Strong Causal Faithfulness Assumption and also allows for uniformly consistent estimates of Markov Equivalence Classes (in a weakened sense), and of identifiable causal effects. 
 
 However, both of these assumptions are restricted to linear Gaussian models. We propose the Generalized k-Triangle Faithfulness, which can be applied to any smooth distribution. In addition, under the Generalized k-Triangle Faithfulness Assumption, we describe the Edge Estimation Algorithm that provides uniformly consistent estimates of causal effects in some cases (and otherwise outputs "can't tell"), and the Very Conservative SGS Algorithm that (in a slightly weaker sense) is a uniformly consistent estimator of the Markov equivalence class of the true DAG.
\section{The Basic Assumptions for Causal Discovery}
\subsection{DAG and Causal Markov Condition}
We use directed acyclic graphs to represent causal relations between variables.  A directed graph $G=\langle \mathbf{V,E}\rangle$ consists of a set of \textit{nodes} \textbf{V} and a set of \textit{edges} $\mathbf{E\subset V\times  V}$. If there is an edge $\langle A,B\rangle \in \mathbf{E}$ , we write $A\rightarrow B$. $A$ is a \textbf{parent} of $B$, and $B$ is a \textbf{child} of $A$, the edge is out of $A$ and into $B$, and $A$ is the source and $B$ is the target.  A \textbf{directed path} from $X$ to $Y$ is a 
sequence of ordered edges where the source of the first edge is $X$, the target of the last edge is $Y$, and if there are n edges in the sequence, for $1\leq i <n$, the target of the $i$th edge is the source of the $i+1$th edge;  $X$ is an \textbf{ancestor}  of $Y$, and $X$ is a \textbf{descendant} of $Y$.   

If a variable $Y$ is in a structure $X\rightarrow Y\leftarrow Z$, and there is no edge between $X$ and $Z$, we call $\langle X,Y,Z\rangle$ an \textit{unshielded collider}; if there is also an edge between $X$ and $Z$,then $\langle X,Y,Z\rangle$ is a triangle and we call $\langle X,Y,Z\rangle$ a \textit{shielded collider}.  If $\langle X,Y,Z\rangle$ is a triangle but $Y$ is not a child of both $X$ and $Z$, we call $\langle X,Y,Z\rangle$ a \textit{shielded non-collider}; if there is no edge between $X$ and $Z$, then $\langle X,Y,Z\rangle$
is a \textit{unshielded non-collider.}

 A \textit{Bayesian network} is an ordered pair $\langle P,G\rangle$ where $P$ is a probability distribution over a set of variables $\textbf{V}$ in $G$. A distribution $P$ over a set of variables $\mathbf{V}$ satisfies the \textbf{ (local) Markov condition} for $G$ if and only if each variable in $\mathbf{V}$ is independent of its non-parents and non-descendants, conditional on its parents.
Given $M=\langle P,G\rangle$, $P_M$ denotes $P$ and $G_M$ denotes $G$.  Two acyclic directed graphs (DAG) $G_1$ and $G_2$ are \textit{ Markov equivalent} if conditional independence relations entailed by Markov condition in $G_1$ are the same as in $G_2$.  It has been proven that two $DAG$s are \textit{Markov equivalent} if and only if they have the same adjacencies and same \textit{unshielded colliders} (Verma and Pearl, 1990). A \textit{pattern} $O$ is an undirected graph that represents a set $M$ of Markov equivalent DAGs: an edge $X\rightarrow Y$ is in $O$ if it is in every DAG in $M$; if $X\rightarrow Y$ is in some DAG and $Y\rightarrow X$ in some other DAG in $M$, then $X-Y$ in $O$  (Spirtes and Zhang, 2016) 

We assume \textbf{causal sufficiency}, which means that $\textbf{V}$ contains all direct common causes of variables in $\textbf{V}$.

\subsection{Faithfulness, linear Gaussian case and \textit{k-}Triangle-Faithfulness}

Given a $\langle P,G\rangle$ that satisfies \textbf{Markov Condition}, we say that $P$ is \textit{faithful} to $G$ if any conditional independence relation that holds in $P$ is entailed by $G$ by the \textbf{Markov Condition}.   We further make the Causal Markov and Faithfulness assumption (Spirtes et  al.,  2001)::

\textbf{Causal Markov Assumption:} If the true causal model $M$ of a population is causally sufficient, each variable in $V$ is independent of its non-parents and non-descendants, conditional on its its parents in $G_M$ (Spirtes and Zhang, 2016). 

\textbf{Causal Faithfulness Assumption: } all conditional independence relations that hold in the population are consequences of the Markov condition from the underlying true causal DAG.

In this paper we talk about cases where $P_M$ over $\mathbf{V}$ for $G=\langle \mathbf{V,E}\rangle$ respects the \textbf{Causal Markov Assumption}.  If $P_M$ is \textit{faithful} to $G_M$ and all variables in $M$ are Gaussian and all causal relations are linear, that is, any $X_i\in\mathbf{V}$ can be written as:

\begin{center}
    $X_i =  \underset{X_j\in Pa_M(X_i)}{\sum} a_{i,j}X_j+\epsilon_j$  
\end{center}  

where $Pa_M(X)$ denotes the set of parents of $X$ in $G_M$, the set of $X$ variables is jointly standard Gaussian, $a_{i,j}$ is a real valued coefficient, and the set of $\epsilon_j$ are jointly Gaussian and jointly independent, conditional correlation between any two variables $\rho_{X,Y|\mathbf{A}}=0$ where $X,Y\in \mathbf{V}$ and $\mathbf{A\subset V}\setminus{\{X,Y}\}$ implies that there is no edge between $X$ and $Y$.  Based on the equation above, we define in the linear Gaussian case the \textit{edge strength} $e_M(X_j\rightarrow X_i)$ as the corresponding coefficient $a_{i,j}$.

It has been proved that under the Causal Markov and Faithfulness assumptions, there are no uniformly consistent estimators of Markov equivalence classes of causal structures represented by DAG (Robins et al., 2003).  Kalisch and B\"{u}hlmann (2007) showed that such uniform consistency is achieved by the $PC$ algorithm if the underlying DAG is sparse relative to the sample size under a strengthened version of Faithfulness Assumption.  This \textbf{Strong Causal Faithfulness Assumption }  in the linear Gaussian case bounds the absolute value of any partial correlation not entailed to be zero by the true causal DAG away from zero by some positive constants. It has the implausible consequence that it puts a lower bound on the strength of edges, since a very weak edge entails a very weak partial correlation. However, the \textbf{Strong Causal Faithfulness Assumption } can be weakened to the strictly weaker (for some values of $k$)  \textbf{\textit{k-}Faithfulness Assumption} while still achieving uniform consistency. Furthermore, at the cost of having a smaller subset of edges oriented, the \textbf{\textit{k-}Faithfulness Assumption} can be weakened to the \textbf{\textit{k-}Triangle-Faithfulness Assumption}, while still achieving uniform consistency and can be relaxed while preserving the uniform consistency: the \textbf{\textit{k-}Triangle-Faithfulness Assumption} (Spirtes and Zhang, 2014) only bounds the conditional correlation between variables in a triangle structure from below by some functions of the corresponding edge strength:

\textbf{\textit{k-}Triangle-Faithfulness Assumption:}  
Given a set of variables \textbf{V}, suppose the true causal model over \textbf{V} is $M=\langle P,G\rangle$, where $P$ is a Gaussian distribution over \textbf{V}, and \textit{G} is a DAG with vertices \textbf{V}. For any variables \textit{X}, \textit{Y}, \textit{Z} that form a triangle in \textit{G}:
\begin{itemize}
    \item if \textit{Y} is a non-collider on the path $\langle X,Y,Z\rangle$, then $|\rho_M(X,Z|\mathbf{W})|\geq K\times |e_M(X-Z)|$ for all $\mathbf{W\subset V}$ that do not contain Y; and
    \item if \textit{Y} is a collider on the path $\langle X,Y,Z\rangle$, then $|\rho_M(X,Z|\mathbf{W})|\geq K\times |e_M(X-Z)|$ for all $\mathbf{W\subset V}$ that do contain Y.
\end{itemize}

where the $X-Z$ represents the edge between $X$ and $Z$ but the direction is not determined.(Sprites and Zhang, 2014)

The \textbf{\textit{k-}Triangle-Faithfulness Assumption} is strictly weaker than the \textbf{Strong Faithfulness Assumption }in several respects: the \textbf{Strong faithfulness Assumption }does not allow edges to be week anyywhere in a graph, while the \textbf{\textit{k-}Triangle-Faithfulness Assumption} only excludes conditional correlations $\rho(X,Z|\mathbf{W})$ from being too small if $X$ and $Z$ are in some triangle structures $\langle X,Y,Z\rangle$ and $X-Z$ is not a weak edge; for every $\epsilon$ used in the \textbf{Strong Faithfulness Assumption }as the lower bound for any partial correlation, there is a $k$ for the \textbf{\textit{k-}Triangle-Faithfulness Assumption} that gives a lower bound smaller than $\epsilon$.

\subsection{VCSGS Algorithm}
The algorithm we use to infer the structure of the underlying true causal graph is \textit{Very Conservative SGS} ($VCSGS$) algorithm, which takes uniformly consistent tests of conditional independence as input:

\textbf{VCSGS Algorithm}
\begin{enumerate}
    \item Form the complete undirected graph $H$ on the given set of variables $\mathbf{V}$.
    \item For each pair of variables $X$ and $Y$ in $\mathbf{V}$, search for a subset $\mathbf{S} $ of $\mathbf{V}\setminus\{X, Y \}$ such that $X$ and $Y$ are independent conditional on $\mathbf{S}$. Remove the edge between $X$ and $Y$ in $H$ if and only if such a set is found.
    \item Let $K$ be the graph resulting from Step 2. For each unshielded triple $\langle X, Y, Z\rangle$ (the only two variables that are not adjacent are $X$ and $Z$ ),
    \begin{enumerate}
        \item If $X$ and $Z$ are not independent conditional on any subset of $\mathbf{V}\setminus{X, Z}$ that contains $Y$, then orient the triple as a collider: $X \rightarrow Y \leftarrow Z$.
        \item If $X$ and $Z$ are not independent conditional on any subset of $\mathbf{V}\setminus{X, Z}$ that does not contain Y, then mark the triple as a non-collider.
        \item Otherwise, mark the triple as ambiguous.
    \end{enumerate}
    \item Execute the following orientation rules until none of them applies:
    \begin{enumerate}
        \item If $X \rightarrow Y-Z$, and the triple $\langle X, Y, Z\rangle$ is marked as a non-collider, then orient $Y-Z$ as $ Y\rightarrow Z$.
        \item If $X \rightarrow Y\rightarrow Z$ and $X-Z$, then orient $X-Z$ as $X\rightarrow Z$.
        \item If $X\rightarrow Y\rightarrow Z$, another triple $\langle X,W,Z\rangle$ is marked as a non-collider, and $W-Y$, then orient $W-Y$ as $W\rightarrow Y$.
    \end{enumerate}
    \item Let $M$ be the graph resulting from step 4. For each consistent disambiguation
of the ambiguous triples in $M$ ($i.e.$, each disambiguation that leads to a
pattern), test whether the resulting pattern satisfies the Markov condition. If
every pattern does, then mark all the `apparently non-adjacent’ pairs as
`definitely non-adjacent’.
\end{enumerate}

It has been proved that under the \textbf{\textit{k-}Triangle-Faithfulness Assumption}, $VCSGS$ algorithm is uniformly consistent in the inference of graph structure. Furthermore, a follow-up algorithm that estimates edge strength given the output of $VCSGS$ also reaches uniform consistency. We are going to prove that the uniform consistency of the estimation of the causal influences under the \textbf{\textit{k-}Triangle-Faithfulness Assumption} can be extended to discrete and nonparametric cases as long as there are uniformly consistent tests of conditional independence (which in the general case requires a smoothness assumption), by showing that missed edges in the inference of causal structure are so weak that the estimations of the causal influences are still uniformly consistent.

\section{Nonparametric Case}
For nonparametric case, we consider variables supported on $[0,1]$.
We first define the strength of the edge $X\rightarrow Y$ as the maximum change in $L1$ norm of the probability of $Y$ when we condition on different values of $X$ while holding everything else constant:

\begin{center}
    
If $x\in Pa(Y):$
$e(X\rightarrow Y) :=max_{pa_{\setminus\{x\}}(Y)\in [Pa(Y)\setminus\{X\}]}max_{x_1,x_2\in[X]}||p_{Y|x_1,pa_{\setminus\{x\}}(Y)}- p_{Y|x_2,pa_{\setminus\{x\}}(Y)}||_1$

\end{center}

where $[X]$ denotes the set of values that $X$ takes,  $[Pa(Y)]\subset [0,1]^{|Pa(Y)|}$ the set of values that parents of Y take, $p_{Y|x_1,pa_{\setminus\{x\}}(Y)}$ the probability distribution of $Y|X=x_1, Pa(Y)\setminus\{X\}=pa_{\setminus\{x\}}(Y)$ and $p_{y|x_1,pa_{\setminus\{x\}}(Y)}$ the density of $p_{Y|x_1,pa_{\setminus\{x\}}(Y)}$ 
for $Y=y$. Since we are conditioning on the set of parents, the conditional probability is equal to the manipulated probability.

  Then we can make the \textbf{\textit{k-}Triangle-Faithfulness Assumption}: given a set of variables \textbf{V}, where the true causal model over \textbf{V} is $M=\langle P,G\rangle$, $P$ is a distribution over \textbf{V}, and \textit{G} is a DAG with vertices \textbf{V}, for any variables \textit{X}, \textit{Y}, \textit{Z} that form a triangle in \textit{G}:
\begin{itemize}
    \item if \textit{Z} is a non-collider on the path $\langle X,Z,Y\rangle$, given any subset $\mathbf{W\subset V}\setminus \{X,Y,Z\}$,\newline $min_{\mathbf{w\in [W]}}min_{x_1, x_2\in[X]}||p_{Y|\mathbf{w},x_1}- p_{Y|\mathbf{w},x_2}||_1\geq K_Y e(X\rightarrow Y)$ for some $K_Y>0$
    \item if \textit{Z} is a collider  on the path $\langle X,Z,Y\rangle$, then for every $y\in [Y]$, given any subset $\mathbf{W\subset V}\setminus \{X,Y,Z\}$,$min_{\mathbf{w\in [W]}}min_{z\in [Z]}min_{x_1,x_2\in[X]}||p_{Y|\mathbf{w}, x_1, z}-p_{Y=y|\mathbf{w},x_2,z}||_1\geq K_Y e(X\rightarrow Y)$ for some $K_Y>0$
\end{itemize}

In order to have uniformly consistent tests of conditional independence, we make smoothness assumption for continuous variables with the support on $[0,1]$:\newline
 \textit{ TV Smoothness(L): } Let $\mathcal{P}_{[0,1],TV(L)}$ be the collection of distributions $p_{Y,\mathbf{A}}$, such that for all $\mathbf{a},\mathbf{a'}\in [0,1]^{|\mathbf{A}|}$, we have:\newline
    $||p_{Y|\mathbf{A=a}}-p_{Y|\mathbf{A=a'}}||_1\leq L||\mathbf{a-a'}||_1$
    
Given the TV smoothness(L), $p$ is continuous.  Furthermore, since $[0,1]^d$ ($d\in \mathbb{N}$) is compact, for any $\mathbf{W,U\subset V}$ (the set of all variables in the true causal graph) , $p_{U|W}$ attains its max and min on its support.  Since $\mathbf{|V|}$ is finite, we can further assume conditional densities are non-zero (NZ(T)):\newline
for any $X\in \mathbf{V}$, $\mathbf{U\subset V}$, $p_{X|\mathbf{U}}\geq T$ for some $1>T>0$.

Notice that by TV Smoothness(L) and that variables have support on $[0,1]$, we can derive an upper bound for probability of any variable given its parents: 

$p_{y|{Pa(Y)=pa(Y)}} \leq ||p_{Y|Pa(Y)=pa(Y)}||_1\newline\leq ||p_{Y|Pa(Y)=pa'(Y)}||_1 + L||pa(Y)-pa'(Y)||_1\newline\leq (1+L|Pa(Y)|) $ 
 
Although the discrete probability case does not have support on $[0,1]$, and its probability is not continuous, it still satisfies the TV smoothness(L) assumption: for instance, if the discrete variables have support only on integers, we can set $L=1$. By replacing the density $p_{X|\mathbf{U}}$ with the probability $P(X|\mathbf{U})$ in NZ(T) assumption, we have a NZ(T) assumption for the discrete case.  Therefore the proof of uniform consistency for the nonparametric case in the rest of the paper also works for the discrete case.

\subsection{Uniform Consistency in the inference of structure}
We use $L1$ norm to characterize dependence :\newline $\epsilon_{X,Y|\mathcal{A}}=||p_{X,Y,\mathcal{A}}-p_{X|\mathcal{A}}p_{Y|\mathcal{A}}p_{\mathcal{A}}||_1$.We want a test $\psi$ of $H_0: \epsilon = 0$ versus $H_1: \epsilon>0 $. $\psi$ is a family of functions: $\psi_0...\psi_n...$ one for each sample size, that takes an i.i.d sample $V_n$ from the joint distribution over $\mathbf{V}$.  Then the test is uniformly consistent w.r.t. a set of distributions $\Omega$ for :
\begin{itemize}
    \item $lim_{n\rightarrow \infty} sup_{P\in\mathcal{P}_{[0,1],TV(L)},\epsilon(P)=0}P^n(\psi_n(V_n)=1)=0$
    \item for every $\delta>0$, $lim_{n\rightarrow \infty} sup_{\epsilon(P)\geq\delta}P^n(\psi_n(V_n)=0)=0$
\end{itemize}

With the TV Smoothness(L) assumption, there are uniformly consistent tests of conditional independence, such as a minimax optimal conditional independence test proposed by Neykov et al.(2020).

Given any causal model $M=lrangle P,G \rangle$ over $\mathbf{V}$, let $C(n,M)$ denote the (random) output of the $VCSGS$ algorithm given an$i.i.d.$ sample of size size $n$ from the distribution $P_M$, then there are three types of errors that it can contain that will mislead the estimation of causal influences:
\begin{enumerate}
    \item $C(n,M)$ \textit{errs in kind I} if it has an adjacency not in $G_M$;
    \item  $C(n,M)$ \textit{errs in kind II} if every adjacency it has is in $G_M$, but it has a marked non-collider not in $G_m$;
    \item  $C(n,M)$ \textit{errs in kind III} if every adjacency and marked non-collider it has is in $G_M$, but it has an orientation not in $G_M$
\end{enumerate}

If $C(n.M)$ errs in either of these three way, there will be variable $X$ and $Y$ in $C(n,M)$ such that $X$ is treated as a parent of $Y$ but is not in the true graph $G_M$; if there is no undirected edge connecting $Y$ in this $C(n,M)$, the algorithm will estimate the causal influence of `` parents'' of $Y$ on $Y$, but such estimation does not bear useful information since intervening $X$ does not really influence $Y$.  Notice that missing an edge is not listed as an mistake here, and we are going to prove later that the estimation of causal influence can still be used to correctly predict the effect of intervention even if the algorithm misses edges.

Let $\phi^{k,L,T}$ be the set of causal models over \textbf{V} under $k$-Triangle-Faithfulness Assumption, TV smoothness(L) and the assumptions of NZ(T). 

We will prove that under the causal sufficiency of the measured variables \textbf{V}, causal Markov assumption, k-Triangle-Faithfulness, TV smoothness(L) assumption and NZ(T) assumption,
\begin{center}
    $\underset{n\rightarrow\infty}{lim}\underset{M\in \phi^{k,L,C}}{sup}P^n_M(C(n,M) errs)=0$
\end{center}

 We begin by proving a useful lemma that bounds $\epsilon_{X,Y|\mathbf{A}}$ with strengths of the edge $X\rightarrow Y$:
\begin{lemma}

Given an ancestral set $\mathbf{A\subset V}$ that contains the parents of $Y$ but not $Y$:

If $X$ is a parent of $Y$: 

$ T^{|\mathbf{A}|} e(X\rightarrow Y)\leq\epsilon_{X,Y|\mathbf{A}\setminus\{X\}}\leq e(X\rightarrow Y)$

\end{lemma}
\begin{proof}
$\epsilon_{X,Y|\mathbf{A}\setminus\{X\}}$\newline$=\int_{\mathbf{A}\setminus\{X\}}\int_X\int_Y|p_{x,y,\mathbf{a}\setminus\{x\}}-\newline p_{y|\mathbf{a}\setminus\{x\}}p_{x|\mathbf{a}\setminus\{x\}}p_{\mathbf{a}\setminus\{x\}}|dydxd\mathbf{a}\setminus\{x\}$

$= \int_{\mathbf{A}} p_{\mathbf{a}}||p_{y|x,\mathbf{A}\setminus\{X\}}-p_{y|\mathbf{A}\setminus\{X\}}||_1d\mathbf{a}$

$\leq \mathbb{E}_{{\mathbf{A}}\sim p_{\mathbf{A}}}[e(X\rightarrow Y)]$

$= e(X\rightarrow Y)$

$\epsilon_{X,Y|\mathbf{A}\setminus\{X\}}$\newline$=\int_{\mathbf{A}\setminus\{X\}}\int_X\int_Y|p_{x,y|\mathbf{a}\setminus\{x\}}-\newline p_{y|\mathbf{a}\setminus\{x\}}p_{x|\mathbf{a}\setminus\{x\}}|dydxd\mathbf{a}\setminus\{x\}$

$= \int_{\mathbf{A}\setminus\{X\}}p_{\mathbf{a}\setminus\{x\}}\int_x p_{x|\mathbf{a}\setminus\{x\}}||p_{Y|x,\mathbf{a}\setminus\{x\}}-\newline p_{Y|\mathbf{a}\setminus\{x\}}||_1dxd\mathbf{a}\setminus\{x\}$

$\geq T^{|\mathbf{A}|} e(X\rightarrow Y)$

The last step is derived using a direct conclusion from NZ(T): 

for any $\mathbf{V\supset W}=\{W_1,W_2...W_n\}$,
by the Chain Rule:
$p_{\mathbf{W}}=\prod_{i=1}^{n}p_{W_i|W_{i+1}...W_n}\geq T^n$

\end{proof}
We are going to prove for each case that the probability for $C(n,M)$ to make each of the three kinds of mistakes uniformly converges to zero.  Since the proofs for the \textit{kind I} and \textit{kind III} errors are almost the same as the proof for the linear Gaussian case (Spirtes and Zhang, 2014), we are only going to prove \textit{kind II} here.

\begin{lemma}
Given causal sufficiency of the measured variables $\mathbf{V}$, the Causal Markov, k-Triangle-Faithfulness, TV smoothness(L) assumption and NZ(T) assumption:

    $\underset{n\rightarrow\infty}{lim}\underset{M\in \phi^{k,L,C}}{sup}P^n_M(C(n,M) \textit{ errs in kind II})=0$

\end{lemma}
\begin{proof}

For any $M\in \phi^{k,L,C}$, if it errs in kind II then it contains a marked non-collider $\langle X,Z,Y\rangle$ that is not in $G_M$.  Since it's been proved  (Spirtes and Zhang, 2014):

    $\underset{n\rightarrow\infty}{lim}\underset{M\in \phi^{k,L,C}}{sup}P^n_M(C(n,M) \textit{ errs in kind I})=0$

the errors of kind II can be one of the two cases:

$(I)$ $\langle X,Z,Y\rangle$ is an unshielded collider in $G_M$;

$(II)$ $\langle X,Z,Y\rangle$ is a shielded collider in $G_M$;

the proof for case (I) is the same as the proof for the $C(n,M)$ errs in kind I (Spirtes and Zhang, 2014), so we are going to prove here that the probability of case (II) uniformly converges to zero as sample size increases.

We are going to prove by contradiction.  Suppose that the probability that $VCSGS$ making a mistake of kind II does not uniformly converge to zero.  Then there exists $\lambda >0$, such that for every sample size $n$, there is a model $M(n)$ such that the probability of $C(n,M(n))$ containing an unshielded non-collider that is  a shielded collider in $M(n)$ is greater than $\lambda$. Let that triangle be $\langle X^{M(n)}, Z^{M(n)}, Y^{M(n)}\rangle$ with $X^{M(n)}$ being the parent of $Y^{M(n)}$ in $M(n)$.

The algorithm will identify the triple as an unshielded non-collider only if:

$(i)$ there is a set $\mathbf{U}^{M(n)}\subset \mathbf{V}^{M(n)}$ containing $Z^{M(n)}$, such that the test of $\epsilon_{X^{M(n)},Y^{M(n)}|\mathbf{U}^{M(n)}} =0$ returns 0, call this test $\psi_{n0}$; 

$(ii)$  there is an ancestral set $\mathbf{W}^{M(n)}$ that contains $X^{M(n)}$ and $Y^{M(n)}$ but not $Z^{M(n)}$, such that for set $\mathbf{A}^{M(n)} = \mathbf{W}^{M(n)}\setminus \{X^{M(n)}, Y^{M(n)}\}$, the test $\epsilon_{X^{M(n)},Y^{M(n)}|\mathbf{A}^{M(n)} }=0$ returns 1, call this test $\psi_{n1}$.

If what we want to proof does not hold for the algorithm,  for all $n$ there is a model $M(n)$:\newline
(1)$P^n_{{M(n)}}(\psi_{n0}=0)>\lambda$\newline
(2)$P^n_{{M(n)}}(\psi_{n1}=1)>\lambda$\newline

(1) tells us that there is some $\delta_n$ such that $|\epsilon_{X^{M(n)},Y^{M(n)}|\mathbf{U}^{M(n)}} |<\delta_n$ and $\delta_n\rightarrow 0$ as $n\rightarrow \infty$ since the test is uniformly consistent.  So we have:\newline
$\delta_n>\epsilon_{X^{M(n)},Y^{M(n)}|\mathbf{U}^{M(n)}  }=\newline \mathbb{E}_{\mathbf{U}^{M(n)}}\sim p_{\mathbf{U}^{M(n)}}[\int_{X^{M(n)}} p_{x^{M(n)}|\mathbf{U}^{M(n)}}||p_{Y^{M(n)}|x^{M(n)},\mathbf{U}^{M(n)}}-\newline p_{Y^{M(n)}|\mathbf{U}^{M(n)}}||_1dx^{M(n)}]\geq\newline \min_{x_1^{M(n)},x_2^{M(n)}\in [X^{M(n)}]} ||p_{Y^{M(n)}|x_1^{M(n)},\mathbf{U}^{M(n)}}-\newline p_{Y^{M(n)}|x_2^{M(n)},\mathbf{U}^{M(n)}}||_1\newline\geq  K_{Y^{M(n)}} e_M(X^{M(n)}\rightarrow Y^{M(n)}) $ \newline 
The last step is by k-Triangle-Faithfulness

So  $e_M(X^{M(n)}\rightarrow Y^{M(n)}) <\dfrac{\delta_n}{K_{Y^{M(n)}}}$.

By Lemma 3.1, $\epsilon_{X^{M(n)},Y^{M(n)}|\mathbf{A}^{M(n)}}<e_M(X^{M(n)}\rightarrow Y^{M(n)})$.  

Therefore, $\epsilon_{X^{M(n)},Y^{M(n)}|\mathbf{A}^{M(n)}}<\dfrac{\delta_n}{K_Y^{M(n)}}\rightarrow0$ as $n\rightarrow \infty$
\end{proof}
\begin{theorem}
Given causal sufficiency of the measured variables $\mathbf{V}$, the Causal Markov, k-Triangle-Faithfulness, TV smoothness(L) and NZ(T) assumptions:

    $\underset{n\rightarrow\infty}{lim}\underset{M\in \phi^{k,L,C}}{sup}P^n_M(C(n,M) \textit{ errs })=0$

\end{theorem}
\begin{proof}
Since we have proved that the probability for $C(n,M)$ to make any of the three kinds of mistakes uniformly converges to 0, the theorem directly follows.
\end{proof}

\subsection{Uniform consistency in the inference of causal effects}

\textbf{Edge Estimation Algorithm: }

For each vertex $Y$ such that all of the edges containing $Y$ are oriented in $C(n,M)$ (output $VCSGS$), if $Pa(Y)$ is the parent set of $Y$ in $C(n,M)$, we use histogram to estimate $ p(y_i|{Pa(Y)=pa(Y)})$\footnote{we denote the density of $p_{Y|Pa(Y)=pa(Y)}$ at $Y=y$ as $p(y|{Pa(Y)=pa(Y)})$ in this section to match with the result of estimation.} for $y_i\in[Y]$ and ${pa(Y)\in [Pa(Y)]}$; for any of the remaining edges, return `Unknown'; if any estimation of conditional probability violates TV smoothness(L), abandon the output and return `Unknown'.

\textbf{Defining the distance between $M_1$ and $M_2$}

The method for estimation for $p(y|{Pa(Y)=pa(Y)})$ is:
we first get $\hat p(Y=y,{Pa(Y)=pa(Y)})$ and  $\hat p({Pa(Y)=pa(Y)})$ by histogram, then we get:

\begin{center}
     $\hat p(y|{Pa(Y)=pa(Y)})=\dfrac{\hat p(Y=y,{Pa(Y)=pa(Y)})}{\hat p({Pa(Y)=pa(Y)})}$
\end{center}

Let $M_1$ be the output of the Edge Estimation Algorithm, and $M_2$ be a causal model, we define the \textit{conditional probability distance}, $d[M_1,M_2]$, between $M_1$ and $M_2$ to be:

    $d[M_1,M_2]$\newline$=\underset{\substack{{Y\in\mathbf{V}},\\{y_i\in[Y]},\\{pa_{M_1}(Y)}\\ {\in [Pa_{M_1}(Y)]},\\{pa_{M_2}(Y)}\\{\in [Pa_{M_2}(Y)]},\\{pa_{M_1}\subset pa_{M_2}}}}{max}|\hat p_{M_1}(y_i|pa_{M_1}(Y))-p_{M_2}(y_i|pa_{M_2}(Y))|$
    
 where $Pa_{M}(Y)$ denotes the parent set of $Y$ in causal model $M$. By convention $|\hat P_{M_1}(y_i|pa_{M_1}(Y))-P_{M_2}(y_i|pa_{M_2}(Y))|=0$ if $\hat P_{M_1}(y_i|pa_{M_1}(Y))$ is ``Unknown".

Now we want to show, the edge estimation algorithm is uniformly consistent in the sense that for every $\delta >0$,
\begin{center}
    $\underset{n\rightarrow\infty}{lim}$ $\underset{M\in \phi^{k,L,C}} {sup}$ $P^n_M(d[\hat O(M),M]>\delta)=0$
\end{center}
Here $M$ is any causal model satisfying causal sufficiency of the measured variables $\mathbf{V}$, the Causal Markov, k-Triangle-Faithfulness, TV smoothness(L) and NZ(T) assumptions and $\hat O(M)$ is the output of the algorithm given an iid sample from $P_M$.

\begin{proof}
Let $\mathcal{O}$ be the set of possible graphs of $VCSGS$.
Since given $\mathbf{V}$, there are only finitely many outputs in $\mathcal{O}$, it suffices to prove that for each output $O\in \mathcal{O}$,
\begin{center}
    
    $\underset{n\rightarrow\infty}{lim}\underset{M\in \phi^{k,L,C}}{sup}P^n_M(d[\hat O(M),M]>\delta|C(n,M)=O)P_M^n(C(n,M)=O)=0$
 
\end{center}
Now partition all the $M$ into three sets given $O$ :\newline
\begin{itemize}
    \item $\Psi_1=\{M|$ all adjacencies, non-adjacencies and orientations in O are true in$ M\}$;
    \item $\Psi_2=\{M|$  only some adjacencies, or orientations in O are not true in $M\}$;
    \item $\Psi_3=\{M|$ only some non-adjacencies in O are not true in $M\}$.
\end{itemize}
It suffices to show that for each $\Psi_i$,
\begin{center}
    $\underset{n\rightarrow\infty}{lim}$ $\underset{M\in \Psi_i} {sup}$ $P^n_M(d[\hat O(M),M]>\delta|C(n,M)=O)P_M^n(C(n,M)=O)=0$
\end{center}
$\Psi_1$:

For any $M\in \Psi_1$, if the conditional probabilities of a vertex $Y$  in $O$ can be estimated (so not ``Unknown"), it means that $Pa_O(Y)=Pa_M(Y)$.  Recall that the histogram estimator is close to the true density with high probability:

for any $\lambda<1$, $sup_{P\in \mathcal{P}_{TV(L)}}P^n(||\hat p_h(x)-p(x)||_{\infty}\leq f(n,\lambda)\leq O((\dfrac{\log n}{n})^{\frac{1}{2+d}}))\geq 1-\lambda$

where $f(n,\lambda)$ is continuous and monotonically decreasing wrt $n$ and $\epsilon$ and $h\propto n^{2/ (2+d)}$ (the number of bins) where $d$ is the dimentionality of the $x$.  For instance, $d=|Pa(Y)|+1$ when estimating $P(Y,Pa(Y))$.

Given a $\delta>0$, $f(n,\epsilon)=\delta$ entails that\newline $sup_{P\in \mathcal{P}_{[0,1],TV(L)}}P^n(||\hat p_h(x)-p(x)||_{\infty}>\delta)<\lambda$. By monotonicity of $f(n,\lambda)$, when $n>n_f $ s.t. $f(n_f,\lambda)=\delta$,\newline $sup_{P\in \mathcal{P}_{[0,1],TV(L)}}P^n(||\hat p_h(x)-p(x)||_{\infty}>\delta)<\lambda$.

Therefore the histogram estimators of \newline$p_M(y,Pa_M(Y)=pa_M(Y))$  and $p_M(Pa_M(Y)=pa_M(Y))$ are uniformly consistent.  Next we are going to use the lemma below: 
\begin{lemma}
 If $\hat p(Y=y,{Pa(Y)=pa(Y)})$and $\hat p({Pa(Y)=pa(Y)})$ are uniformly consistent\newline estimators of  $ p(Y=y,{Pa(Y)=pa(Y)})$and \newline  $ p({Pa(Y)=pa(Y)})$,
then

     $\hat p(y|{Pa(Y)=pa(Y)})=\dfrac{\hat p(Y=y,{Pa(Y)=pa(Y)})}{\hat p({Pa(Y)=pa(Y)})}$

is a uniformly consistent estimator for\newline $p(y|{Pa(Y)=pa(Y)})$

\end{lemma}

\begin{proof}

Recall that:

for any $\lambda<1$, $sup_{P\in \mathcal{P}_{TV(L)}}P^n(||\hat p_h(x)-p(x)||_{\infty}\leq f(n,\lambda)= O((\dfrac{\log n}{n})^{\frac{1}{2+d}}))\geq 1-\lambda$ $(*)$

where $f(n,\lambda)$ is continuous and monotonically decreasing wrt $n$ and $\lambda$ and $h\propto n^{2/ (2+d)}$ (the number of bins) where $d$ is the dimentionality of the $x$.  For instance, $d=|Pa(Y)|+1$ when estimating $P(Y,Pa(Y))$.

For any $\delta>0$, $f(n,\lambda)=\delta$ entails that\newline $sup_{P\in \mathcal{P}_{[0,1],TV(L)}}P^n(||\hat p_h(x)-p(x)||_{\infty}>\delta)<\lambda$. By monotonicity of $f(n,\lambda)$, when $n>n_f $ s.t. $f(n_f,\lambda)\leq\delta$, $sup_{P\in \mathcal{P}_{[0,1],TV(L)}}P^n(||\hat p_h(x)-p(x)||_{\infty}>\delta)<\lambda$.

Let $d =   |Pa(Y)|$,   notice that for any $\lambda$, with probability at least $1-\lambda$,

$|\dfrac{ \hat p(Y=y,{Pa(Y)=pa(Y)})}{ \hat p({Pa(Y)=pa(Y)})}-p(y|{Pa(Y)=pa(Y)})|$\footnote{here we use $\hat p (x)$ instead of $\hat p_h (x)$ because $h$ is dependent on the dimension of $x$}

$=\dfrac{1}{\hat p({Pa(Y)=pa(Y)})}| \hat p(Y=y,{Pa(Y)=pa(Y)})-p(y|{Pa(Y)=pa(Y)})\hat p({Pa(Y)=pa(Y)})| $

$\leq\dfrac{1}{T^{d}}| p(Y=y,{Pa(Y)=pa(Y)})+\newline O((\dfrac{\log n}{n})^{\frac{1}{3+d}})-p(y|{Pa(Y)=pa(Y)})( p({Pa(Y)=pa(Y)})-O((\dfrac{\log n}{n})^{\frac{1}{2+d}}))|$ (By $(*)$ and the fact that the estimation result can only be valid  if it satisfies TV smoothness(L))\footnote{Recall that $\mathbf{V}$ denotes the set of variables in the true graph}

$\leq \dfrac{1}{T^{d}}|O((\dfrac{\log n}{n})^{\frac{1}{3+d}})+O((\dfrac{\log n}{n})^{\frac{1}{2+d}})|$\newline

$= O((\dfrac{\log n}{n})^{\frac{1}{3+d}})$

The second to the last step is derived because $p(y|{Pa(Y)=pa(Y)})$is upper bounded by $(1+L|Pa(Y)|) $ by TV smoothness.

We have:

$sup_{P\in \mathcal{P}_{TV(L)}}P^n(|\dfrac{ \hat p(Y=y,{Pa(Y)=pa(Y)})}{ \hat p({Pa(Y)=pa(Y)})}-$
\newline $p(y|{Pa(Y)=pa(Y)})|\leq O((\dfrac{\log n}{n})^{\frac{1}{3+d}}))\geq 1-\lambda$

So 
     $\hat p(y|{Pa(Y)=pa(Y)})=$\newline$\dfrac{\hat p(Y=y,{Pa(Y)=pa(Y)})}{\hat p({Pa(Y)=pa(Y)})}$
is a uniform consistent estimator for $p(y|{Pa(Y)=pa(Y)})$

\end{proof}

By lemma 3.4, we conclude that the $\hat p_M(y|Pa_M(Y)=pa_M(Y))$ is a uniformly consistent estimator for
\begin{center}
$p_M(y|Pa_M(Y)=pa_M(Y))=\frac{p_M(y,Pa_M(Y)=pa_M(Y))}{p_M(Pa_M(Y)=pa_M(Y))}$,  
\end{center}
So:

    $\underset{n\rightarrow\infty}{lim}\underset{M\in \Psi_1} {sup}P^n_M(d[\hat O(M),M]>\delta|C(n,M)=$\newline$O)P_M^n(C(n,M)=O)\newline\leq\underset{n\rightarrow\infty}{lim}\underset{M\in \Psi_1} {sup}P^n_M(d[\hat O(M),M]>\delta|C(n,M)=O)=0$

$\Psi_2:$ the proof is exactly the same as for the discrete case.\newline
$\Psi_3:$

Let $O(M)$ be the population version of $\hat O(M)$.  Since the histogram estimator is uniformly consistent over $||\hat p_h-p||_\infty$ and there are finitely many parent-child combinations, for every $\lambda >0$ there is a sample size $N_1$, such that for $n>N_1$, and all $M\in \Psi_3$,
\begin{center}
     $P^n_M(d[\hat O(M),O(M)]>\delta/2|C(n,M)=O)<\lambda$
\end{center}

Since only some non-adjacencies in $O$ are not true in $M$, we know that for any vertex $Y$ that have some estimated conditional probabilities given its parents in $O$, $Pa_{O(M)}(Y)\subset Pa_M(Y)$ where $Pa_{O(M)}(Y)$ denotes the parent set of $Y$ in the $O$ when the underlying probability is $P_M$(i.e., $M$ is the true causal model).  Since $Pa_M(Y)\not\subset Pa_{O(M)}(Y)$, for any $y_i\in [Y]$ and $pa_{O(M)}(Y)\in [Pa_{O(M)}(Y)]$, $P_O(y_i|Pa_{O(M)}(Y)=pa_{O(M)}(Y))$ is a marginalization of $p_M(y_i|Pa_M(Y)=pa_M(Y))$. Therefore, the distance between $O(M)$ and $M$ is:

    $d[O(M),M]=\newline\underset{\substack{{Y\in\mathbf{V}},\\{y_i\in[Y]},\\ {pa_{O(M)}(Y)\in}\\{[Pa_{O(M)}(Y)]},\\ {pa_{M}(Y)\in}\\{ [Pa_{M}(Y)]},\\ {pa_{O(M)}\subset pa_{M}}}}{max}| p_{O(M)}(y_i|pa_{O(M)}(Y))-p_{M}(y_i|pa_{M}(Y))|$

Given the $Y$ corresponding to the equation above, let $Pa_M(Y)=\{A_1...A_g..A_{g+h}\}$ and $Pa_{O(M)}(Y)=\{A_1...A_g\}$.
Since $P_{O(M)}(y_i|pa_{O(M)}(Y))$ is the marginalization of all $P_M(y_i|pa_M(Y))$, we have:\newline

    $| p_{O(M)}(y_i|pa_{O(M)}(Y))-p_{M}(y_i|pa_{M}(Y))|\leq $
    
    $\underset{\substack{pa_{M}(Y)_1,\\pa_{M}(Y)_2 \\ \in
    [Pa_{M}(Y)], \\s.t. {pa_{O(M)}\subset} \\{ pa_{M}(Y)_1\cap pa_{M}(Y)_2}}}{max}| p_{M}(y_i|pa_{M}(Y)_{1})-p_{M}(y_i|pa_{M}(Y)_{2})|$
    
    $<\sum_{j=1}^{j=h}e_M(A_{g+j}\rightarrow Y)$

If $ \sum_{j=1}^{j=h}e_M(A_{g+j}\rightarrow Y)< \delta/2$, then we have:
\begin{center}
    $d[O(M),M]<\delta/2$
\end{center}
For all such $M$, there is a $N_1$, such that for any $n>N_1$:

    $P^n_M(d[\hat O(M),M]>\delta|C(n,M)=O)$\newline
    $\leq P^n_M(d[\hat O(M),O(M)]+d[ O(M),M]>\delta|C(n,M)=O)$\newline
     $\leq P^n_M(d[\hat O(M),O(M)]>\delta/2|C(n,M)=O)<\epsilon$

If $\sum_{j=1}^{j=h}e_M(A_{g+j}\rightarrow Y)\geq \delta/2$, then there is at least an $w\in \{1,2...h\}$, $s,t.$ $e_M(A_{g+w}\rightarrow Y)>\dfrac{\delta}{2h}$. By Lemma 3.1:
\begin{center}
    $\epsilon_{A_{g+w},Y|\mathbf{U} }\geq T^{|\mathbf{U}|+1}e_M(A_{g+w}\rightarrow Y)>T^{|\mathbf{U}|+1}\dfrac{\delta}{2h}$. 
\end{center}
where the $\mathbf{U}\cup \{A_{g+w},Y\}$ is some ancestral set not containing any descendant of $Y$. 

Since the density estimation does not turn ``unknown", we know that in step 5 of $VCSGS$ the test of $\epsilon_{A_{g+w},Y|U}=0$ returns 0 while $\epsilon_{A_{g+w},Y|U}\geq T^{|\mathbf{U}|+1}\dfrac{\delta}{2h}$.  Since the test is uniformly consistent, it follows that there is a sample size $N_2$ such that
\begin{center}
$P^{N_2}_M(\epsilon_{A_{g+w},Y|U}=0)<\lambda$
\end{center}
 for any $n>N_2$ and therefore for all such M,
\begin{center}
    $P^n_M(d[\hat O(M),M]>\delta|C(n,M)=O)<\lambda$
\end{center}
Let $N = max(N_1, N_2)$, for $n>N$,
 
        $\underset{n\rightarrow\infty}{lim}\underset{M\in \Psi_3} {sup}P^n_M(d[\hat O(M),M]>\delta|C(n,M)=$\newline$O)P_M^n(C(n,M)=O)\newline\leq \underset{n\rightarrow\infty}{lim}\underset{M\in \Psi_3} {sup} P^n_M(d[\hat O(M),M]>\delta|C(n,M)=O)=0$

\end{proof}

\section{Discussion}

We have shown that there is a uniformly consistent estimator of causal structure and some causal effects for nonparametric distributions under the  $k$-Triangle-Faithfulness assumption, which is sometimes stronger than the Faithfulness Assumption and weaker than the Strong Faithfulness Assumption.  There are a number of open questions, such as whether the causal sufficiency assumption can be relaxed, so we allow the existence of latent variables and whether there are similar results under assumptions weaker than the Causal Faithfulness Assumption, such as the Sparsest Markov Representation Assumption (Solus et al.,2016).

\end{document}